\documentclass{article}

\usepackage{times}
\usepackage{graphicx}
\usepackage{subfigure} 
\usepackage{natbib}
\usepackage{hyperref}

\usepackage[accepted]{icml2017}

\usepackage{amsmath}
\usepackage{amsfonts}

\newtheorem{definition}{Definition}
\newtheorem{lemma}{Lemma}
\newtheorem{theorem}{Theorem}

\newtheorem{corollary}{Corollary}

\def\n(#1){\bar{#1}}

\def\pr{{\it Pr}}

\def\X{{\bf X}}
\def\x{{\bf x}}
\def\Y{{\bf Y}}
\def\y{{\bf y}}
\def\Z{{\bf Z}}
\def\z{{\bf z}}

\def\eql(#1,#2){{#1\!=\!#2}}
\newcommand\name[1]{\ensuremath{\mathsf{#1}}}

\newcommand{\argmax}{\operatornamewithlimits{arg max}}

\def\eproof{{\hfill$\Box$}}
\newenvironment{proof}[1][Proof]
			   {\begin{trivlist}\item[\hskip \labelsep {\bfseries #1}]}
			   {~\eproof \end{trivlist}}

\def\ace{\name{ace}}
\def\ac{{\cal AC}}
\def\vars{{\it vars}}

\newcommand\de[1]{\underline{#1}}
			   
\icmltitlerunning{On Relaxing Determinism in Arithmetic Circuits}

\begin{document} 

\twocolumn[
\icmltitle{On Relaxing Determinism in Arithmetic Circuits}

\icmlsetsymbol{equal}{*}

\begin{icmlauthorlist}
\icmlauthor{Arthur Choi}{ucla}
\icmlauthor{Adnan Darwiche}{ucla}
\end{icmlauthorlist}

\icmlaffiliation{ucla}{University of California, Los Angeles, California, USA}
\icmlcorrespondingauthor{Arthur Choi}{aychoi@cs.ucla.edu}
\icmlcorrespondingauthor{Adnan Darwiche}{darwiche@cs.ucla.edu}

\icmlkeywords{tractable probabilistic models, arithmetic circuits, knowledge compilation, ICML}

\vskip 0.3in
]

\printAffiliationsAndNotice{}

\begin{abstract}
The past decade has seen a significant interest in learning tractable probabilistic representations. Arithmetic circuits (ACs) were among the first proposed tractable representations, with some subsequent representations being instances of ACs with weaker or stronger properties.  In this paper, we provide a formal basis under which variants on ACs can be compared, and where the precise roles and semantics of their various properties can be made more transparent.  This allows us to place some recent developments on ACs in a clearer perspective and to also derive new results for ACs. This includes an exponential separation between ACs with and without determinism; completeness and incompleteness results; and tractability results (or lack thereof) when computing most probable explanations (MPEs).
\end{abstract}

\section{Introduction}

Arithmetic circuits (ACs) were introduced into the AI literature more than fifteen years ago as a tractable representation of probability distributions. The original work on these circuits proposed the compilation of such circuits from Bayesian networks, while identifying and assuming three circuit properties, called determinism, decomposability and smoothness \cite{darwiche01tractable,darwicheJACM-DNNF,Darwiche03}. Since then, the literature on using arithmetic circuits for probabilistic reasoning has seen two key developments.  The first is the proposal made by \cite{LowdD08} to learn these circuits directly from data---instead of just compiling them from models---therefore creating two distinct construction modes for these circuits. The second development, reported by \cite{PoonD11}, amounted to proposing a class of arithmetic circuits that does not satisfy determinism, under the name of sum-product networks (SPNs).

An examination of the literature surrounding arithmetic circuits and their variants suggests that the implications of relaxing determinism are not very well understood, even leading to conflicting claims in some cases. The treatment of smoothness has also not been very consistent as far as its necessity for certain operations on arithmetic circuits, and the complexity of enforcing it. Our goal in this paper is to address some of these issues by providing a systematic and formal treatment of arithmetic circuits, while focusing on the precise roles and semantics of their various properties and the implications of relaxing determinism.

We make several contributions in this paper. We start by reconstructing the original definition of arithmetic circuits given in \cite{darwicheKR02,Darwiche03}, while assuming that these circuits represent arbitrary factors, instead of just distributions induced by Bayesian networks (a particular class of factors). We then provide definitions for decomposability, smoothness and determinism in the context of this reconstruction, while isolating precisely the role that each one of these properties play. Some of what we report on this has already been observed in the literature, but we provide alternate or more formal proofs for the sake of a systematic and inclusive treatment. We also derive new results. The first of these is a separation theorem showing that relaxing determinism can lead to exponentially smaller arithmetic circuits, while preserving the ability of these circuits to compute marginals in linear time. This begs the question of whether anything is lost from relaxing determinism. On this front, we highlight a finding that has already been reported in the literature and introduce new ones. In particular, we provide an expanded proof for the observation that relaxing determinism deprives arithmetic circuits from the ability to compute MPE in linear time. We also add a new result showing that enforcing decomposability has the power of solving MPE, even though the MPE query is not tractable for decomposable circuits. Moreover, we show that relaxing determinism leads to a type of incompleteness that we call {\em parametric incompleteness,} with important implications on the compilability of circuits from models. Our final contribution is a formal correctness proof of the linear-time MPE algorithm, originally proposed by \cite{ChanD06}, but with respect to the reconstructed definition of arithmetic circuits satisfying decomposability, determinism and smoothness.

This paper is structured as follows. We reconstruct the definition of arithmetic circuits as given by \cite{darwicheKR02,Darwiche03} in Section~\ref{sec:ACs}, but with respect to factors instead of distributions (of Bayesian networks). We then provide a new treatment of decomposability and smoothness in Section~\ref{sec:decomposability}, followed by a new treatment of determinism in Section~\ref{sec:determinism}. We finally focus on the relaxation of determinism in Section~\ref{sec:relax}, where we provide a new set of results and insights.
%An extended version of this paper includes some proofs that have been omitted here for space limitations.\footnote{Available at \url{http://reasoning.cs.ucla.edu}.}

\section{Arithmetic Circuits} \label{sec:ACs}
Capital letters (\(X\)) denote variables and lower-case letters (\(x\)) denote their values. Bold capital letters (\(\X\)) denote sets of variables and bold lower-case letters (\(\x\)) denote their instantiations. Value \(x\) is {\em compatible} with instantiation \(\y\) iff \(\y\) assigns value \(x\) to \(X\) or does not assign any value to \(X\).
\begin{definition}
A \de{factor} \(f(\X)\) over variables \(\X\) maps each instantiation \(\x\) of variables \(\X\) into a non-negative number \(f(\x)\). The factor is a distribution when \(\sum_\x f(\x) = 1\).
\end{definition}
The classical, tabular representation of a factor \(f(\X)\) is clearly exponential in the number of variables \(\X\), yet it allows one to answer key probabilistic queries efficiently. The interest here is in a more compact representation of these factors, using arithmetic circuits, while preserving the ability to answer some of these queries efficiently. We focus on the following queries, all with respect to a factor \(f(\X)\), with its variables \(\X\) partitioned into sets \(\Y\) and \(\Z\):
\begin{itemize}
\item Computing the {\em value} of factor \(f(\X)\) at instantiation \(\y\), defined as \(f(\y) = \sum_{\z} f(\y,\z)\) and called a {\em marginal} in this paper. This corresponds to the \emph{probability of evidence} \(\y\) when the factor is a distribution.
\item Computing the {\em MPE} of factor \(f(\X)\), defined as \(\argmax_\x f(\x)\). This corresponds to the most likely instantiation when the factor is a distribution. 
\item Computing the {\em MAP} over variables \(\Y\), defined as  \(\argmax_\y \sum_\z f(\y,\z)\). This is the most likely state of variables \(\Y\) when the factor is a distribution. 
\end{itemize}
For Bayesian networks (interpreted as factors), the decision variants of the MPE, marginals, and MAP problems are, respectively, \(\mathrm{NP}\)-complete \cite{Shimony94}, \(\mathrm{PP}\)-complete \cite{Roth96}, and \(\mathrm{NP}^{\mathrm{PP}}\)-complete \cite{ParkD04}; see also \cite{Darwiche09}. Hence, computing marginals is more difficult than computing MPE---an observation that will be quite relevant later.

We also need to define the {\em projection} of factor \(f(\X)\) on variables \(\Y\) as the factor \(g(\Y)\) with \(g(\y) = \sum_\z f(\y,\z)\). This projection will be denoted by \(\sum_\Z f(\X)\).

We next define an arithmetic circuit over discrete variables \(\X\), as utilized in~\cite{darwicheKR02,Darwiche03} to represent distributions, except that we will utilize it to represent factors. A key observation here is that the circuit inputs are not variables \(\X\), but {\em indicator variables} that are derived from the values of variables \(\X\)~\cite{darwicheKR02,Darwiche03}.
\begin{definition}\label{def:ac}
An \de{arithmetic circuit} \(\ac(\X)\) over variables \(\X\) is a rooted DAG whose internal nodes are labeled with \(+\) or \(*\) and whose leaf nodes are labeled with either indicator variables \(\lambda_x\) or non-negative parameters \(\theta\). Here, \(x\) is the value of some variable \(X\) in \(\X\). The \de{value} of the circuit at instantiation \(\y\), where \(\Y \subseteq \X\), is denoted \(\ac(\y)\) and obtained by assigning indicators \(\lambda_x\) the value \(1\) if \(x\) is compatible with instantiation \(\y\) and \(0\) otherwise, then evaluating the circuit in the standard way (in linear time).
\end{definition}

The following definition makes a distinction that has not been made explicit in the literature as far as we know, but is critical for a clear semantics of arithmetic circuits.
\begin{definition}\label{def:computes}
The circuit \(\ac(\X)\) \de{computes factor} \(f(\X)\) iff \(\ac(\x)=f(\x)\) for each instantiation \(\x\) of variables \(\X\). The circuit \de{computes the factor marginals} iff \(\ac(\y) = f(\y)\) for each instantiation \(\y\) of every \(\Y \subseteq \X\).
\end{definition}

The notion of ``computes a factor'' constrains the value of an arithmetic circuit under a strict subset of its inputs (i.e., those corresponding to complete instantiations). However, the notion of ``computes marginals'' constrains its value under every input. Hence, two arithmetic circuits that represent distinct functions (over indicator variables) may still compute the same factor. Consider an arithmetic circuit that computes a factor \(f(X,\ldots)\), where \(X\) has values \(x\) and \(\bar{x}\). Replacing \(\lambda_x + \lambda_{\bar{x}}\) with \(1\) in this circuit preserves its ability to compute the factor since \(\lambda_x + \lambda_{\bar{x}} = 1\) for every input that is relevant to computing the factor. This replacement, however, will change the function represented by the circuit and its ability to compute the factor marginals.

\begin{figure}[tb]
\centering
\subfigure[\(f_1(A)\)]{\label{fig:factor1}
\small
\begin{tabular}[b]{c|c}
\(A\) & \(f_1\) \\\hline
1 & 1 \\
0 & 2
\end{tabular}}
\quad \quad
\subfigure[\(f_2(A,B)\)]{\label{fig:factor2}
\small
\begin{tabular}[b]{cc|c}
\(A\) & \(B\) & \(f_2\) \\\hline
1 & 1 & 3\\
1 & 0 & 4 \\
0 & 1 & 5 \\
0 & 0 & 6 
\end{tabular}}
\quad \quad
\subfigure[\(f = f_1 f_2\)]{\label{fig:factor3}
\small
\begin{tabular}[b]{cc|c}
\(A\) & \(B\) & \(f\) \\\hline
1 & 1 & 3\\
1 & 0 & 4 \\
0 & 1 & 10 \\
0 & 0 & 12
\end{tabular}}
\caption{Two factors and their product.
\label{fig:factors}}
\end{figure}

\begin{corollary}
An arithmetic circuit that computes the marginals of a factor also computes the factor. However, an arithmetic circuit that computes a factor does not necessarily compute its marginals.
\end{corollary}
Consider the following arithmetic circuit which computes the factor in Figure~\ref{fig:factor3}:
\[
[\lambda_{a} + 2\lambda_{\bar{a}}]
[3\lambda_{a} \lambda_{b} + 
4\lambda_{a} \lambda_{\bar{b}} + 
5\lambda_{\bar{a}} \lambda_{b} + 
6\lambda_{\bar{a}} \lambda_{\bar{b}}].
\]
This circuit {\em does not} compute the factor marginals. Moreover, this circuit is the product of two circuits, one computing factor \(f_1\), the other computing factor \(f_2\), as in Figure~\ref{fig:factors}.

To get further insights into the notion of ``computing marginals,'' we appeal to the notion of a network polynomial~\cite{Darwiche03}, while lifting it to factors.
\begin{definition}\label{def:poly}
The \de{polynomial} of factor \(f(\X)\) is defined as:
\[
\sum_\x f(\x) \prod_{x \sim \x} \lambda_x,
\]
where \(x \sim \x\) means that the value \(x\) of variable \(X \in \X\) is compatible with instantiation \(\x\) of variables \(\X\).
\end{definition}
The polynomial of factor \(f(A,B)\) in Figure~\ref{fig:factor3} is
\(
3\lambda_{a} \lambda_{b} + 
4\lambda_{a} \lambda_{\bar{b}} + 
10\lambda_{\bar{a}} \lambda_{b} + 
12\lambda_{\bar{a}} \lambda_{\bar{b}}
\).
The polynomial of factor \(f(\X)\) corresponds to the simplest circuit that computes the factor marginals. It is a two-level circuit though, which has an exponential size. The interest, however, is in deeper circuits that may not be exponentially sized. We later discuss circuit properties that allows us to achieve this, sometimes.

One can construct an arithmetic circuit that computes the distribution of a Bayesian network or the partition function of a Markov network in time and space that are linear in the size of these models. Each of these models correspond to a set of factors \(f_1, \ldots, f_n\), with the model representing the product of these factors. We can construct a circuit that computes each factor as given in Definition~\ref{def:poly}, then simply combine these circuits using a multiplication node. The result will compute the factor \(f = f_1 \ldots f_n\) but it will not necessarily compute its marginals. We next show that if we enforce the properties of decomposability and smoothness on such a circuit, while maintaining its ability to compute the factor \(f\), the resulting circuit will also compute the factor marginals. Therefore, these two properties turn the circuit into a tractable representation of the factor, allowing one to compute marginals by simply evaluating the circuit as in Definition~\ref{def:ac} (in time linear in the circuit size).

\section{Decomposability and Smoothness} \label{sec:decomposability}

\begin{figure}[tb]
\centering
\includegraphics[width=.35\linewidth]{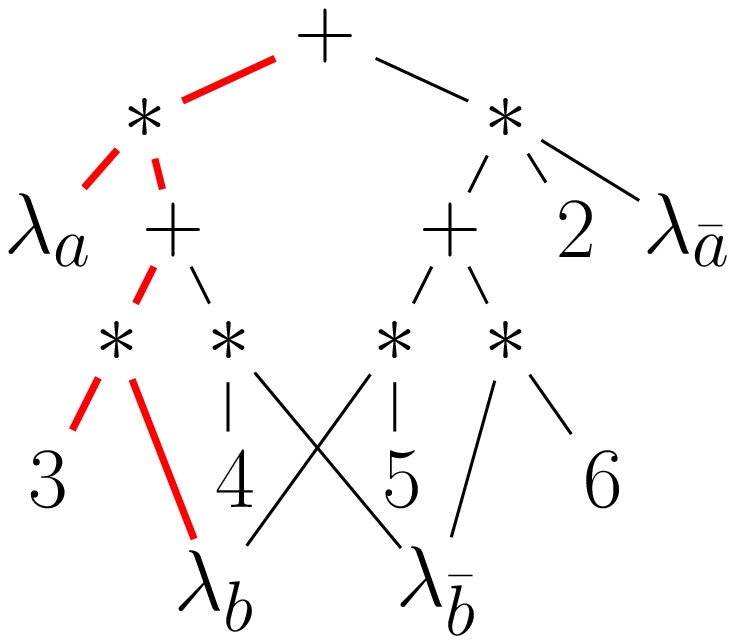}
\quad
\includegraphics[width=.35\linewidth]{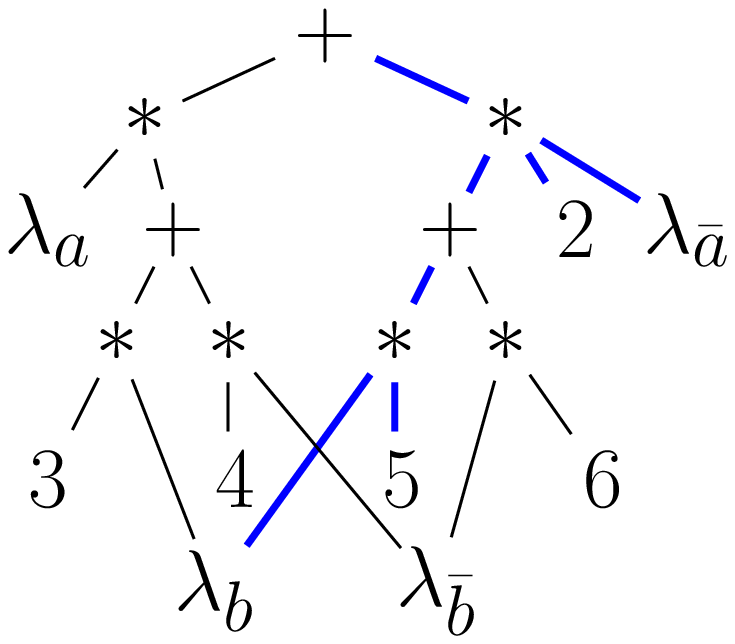}
\caption{An AC that computes factor \(f = f_1 \cdot f_2\) from Figure~\ref{fig:factors}, where an \(ab\)-subcircuit (left) and an \(\n(a)b\)-subcircuit (right) are highlighted. This circuit is deterministic, decomposable and smooth.
\label{fig:ac}}
\end{figure}

The property of {\em decomposability} \cite{darwicheJACM-DNNF} was used for tractable probabilistic reasoning in~\cite{Darwiche03} by compiling Bayesian networks into arithmetic circuits that are guaranteed to be decomposable. This property was also enforced by the algorithm proposed in~\cite{LowdD08} for learning arithmetic circuits.
\begin{definition}[Decomposability]
Let \(n\) be a node in an arithmetic circuit \(\ac(\X)\). The variables of \(n\), denoted \(\vars(n)\), are all variables \(X \in \X\) with some indicator \(\lambda_x\) appearing at or under node \(n\). An arithmetic circuit is \de{decomposable} iff every pair of children \(c_1\) and \(c_2\) of a \(*\)-node satisfies \(\vars(c_1)\cap\vars(c_2) = \emptyset\).
\end{definition}

The property of {\em smoothness} \cite{darwicheJACM-DNNF} was also used for probabilistic reasoning in~\cite{Darwiche03} by compiling circuits that are smooth. It was also enforced by the learning algorithm of~\cite{LowdD08}. This property was later called {\em completeness} in the works on SPNs, initially in~\cite{PoonD11}.

\begin{definition}[Smoothness]
An arithmetic circuit \(\ac(\X)\) is \de{smooth} iff 
(1) it contains at least one indicator for each variable in \(\X\), and 
(2) for every child \(c\) of \(+\)-node \(n\), we have \(\vars(n) = \vars(c)\).
\end{definition}
Consider a factor \(f(X)\), where variable \(X\) is binary and \(f(x) = f(\bar{x}) = \theta\). A circuit that consists of the single parameter \(\theta\) will compute this factor but is not smooth. The circuit \(\theta \lambda_x + \theta \lambda_{\bar{x}}\) computes this factor and is smooth.

Consider a variable \(X\) with values \(x_1, \ldots, x_m\). Multiplying a circuit node by \(\lambda_{x_1} + \ldots + \lambda_{x_m}\) preserves the circuit's ability to compute a given factor since \(\lambda_{x_1} + \ldots + \lambda_{x_m} = 1\) under any circuit input that is relevant to this computation. One can use this technique to ensure the smoothness of any circuit, while incurring only a polynomial overhead.\footnote{\cite{darwiche01tractable} shows how to smooth an NNF circuit in \(O(mn)\) space and time, where \(n\) is the size of the circuit and \(m\) is the number of variables (the method can be adapted to ACs).\label{footnote:smooth}} Hence, contrary to decomposability and determinism, enforcing smoothness is not difficult computationally, yet it is necessary for an arithmetic circuit to compute marginals as we discuss later. We also state the following observation, used extensively in inductive proofs that we utilize later.

\begin{lemma}\label{lem:induction}
Consider a decomposable and smooth arithmetic circuit \(\ac(\X)\) and define \(\X_n = \vars(n)\) for each circuit node \(n\). Each arithmetic circuit \(\ac(\X_n)\) rooted at node \(n\) is also decomposable and smooth.
\end{lemma}

A main insight in this paper is the use of subcircuits, first introduced in \cite{ChanD06} for a different purpose. They were also adopted in \cite{DennisV15,ZhaoPG16} to motivate SPN learning algorithms.
\begin{definition}[Subcircuits] \label{def:subcircuit}
A \de{complete subcircuit} \(\alpha\) of an arithmetic circuit is obtained by traversing the circuit top-down, while choosing one child of each visited \(+\)-node and all children of each visited \(*\)-node. The \de{term} of \(\alpha\) is the set of values \(x\) for which indicator \(\lambda_x\) appears in \(\alpha\). The \de{coefficient} of \(\alpha\) is the product of all parameters in \(\alpha\).
\end{definition}

The circuit \(2*\lambda_{x} + 1*\lambda_{\bar{x}} + 3*\lambda_{x}\) computes factor \(f(X)\) with \(f(x)=5\) and \(f(\bar{x})=1\). It is decomposable, smooth and has three complete subcircuits, with \((x,2)\), \((\bar{x},1)\) and \((x,3)\) as their term/coefficient pairs. Note that two subcircuits may have the same term but different coefficients.

The following lemma and its proof reveal the precise roles of decomposability and smoothness. Given decomposability, the term of a complete subcircuit will not contain conflicting values for any variable. Given smoothness, the term must contain a value for each variable.

\begin{lemma}\label{lem:inst}
Let \(\alpha\) be a complete subcircuit of an arithmetic circuit \(\ac(\X)\) that is decomposable and smooth. The term of subcircuit \(\alpha\) must be an instantiation of variables \(\X\).
\end{lemma}

\begin{proof}
Given smoothness, every variable \(X \in \X\) must have at least one indicator \(\lambda_x\) in \(\alpha\) (no variables are dropped from the circuit if we keep only a single child of a \(+\)-node). Given decomposability, each variable \(X \in \X\) must have at most one indicator \(\lambda_x\) in \(\alpha\). Hence, \(\alpha\) will contain exactly one indicator for each variable \(X \in \X\). The term of \(\alpha\) is therefore an instantiation \(\x\) of variable \(\X\). 
\end{proof}
A complete subcircuit with term \(\x\) will be called an \de{\(\x\)-subcircuit.} Figure~\ref{fig:ac} depicts an \(ab\)-subcircuit (in red) and an \(\bar{a}b\)-subcircuit (in blue). 

In a decomposable and smooth circuit \(\ac(\X)\), every complete subcircuit is an \(\x\)-subcircuit for some instantiation \(\x\) of variables \(\X\). The circuit can then be treated as a collection of \(\x\)-subcircuits (multiple subcircuits can have the same term). Our proofs utilize this implication heavily.

\begin{definition}
An \de{input} \(\Lambda\) for arithmetic circuit \(\ac(\X)\) assigns a value in \(\{0,1\}\) to each circuit indicator \(\lambda_x\). An instantiation \(\x\) is \de{compatible} with a circuit input \(\Lambda\), denoted \(\x \sim \Lambda\), iff the input sets \(\lambda_x = 1\) when \(\x\) sets \(\eql(X,x)\).
\end{definition}
A circuit input can be viewed as the set of instantiations compatible with it. Consider the binary variables \(\X = \{A,B,C\}\) for an example. The circuit input
\[
\Lambda = 
\{\lambda_a = 1,\lambda_{\n(a)} = 0, \lambda_b = 1, \lambda_{\n(b)} = 0, \lambda_c = 1, \lambda_{\n(c)} = 0\}
\]
has a single compatible instantiation \(abc\). The input
\[
\Lambda = 
\{\lambda_a = 0,\lambda_{\n(a)} = 0, \lambda_b = 1, \lambda_{\n(b)} = 0, \lambda_c = 1, \lambda_{\n(c)} = 0\}
\]
has no compatible instantiations, and the circuit input:
\[
\Lambda = 
\{\lambda_a = 1,\lambda_{\n(a)} = 1, \lambda_b = 1, \lambda_{\n(b)} = 0, \lambda_c = 1, \lambda_{\n(c)} = 0\}
\]
has two compatible instantiations \(abc\) and \(\n(a)bc\).  In this latter case, evaluating the circuit at instantiation \(bc\), as discussed in Definition~\ref{def:ac}, leads to evaluating it under input \(\Lambda\).

The following lemma brings us one step away from showing why decomposability and smoothness force an AC that computes a factor to also compute the factor marginals.

\begin{lemma}\label{lem:sum}
Given a decomposable and smooth arithmetic circuit, let \(\theta_1, \ldots, \theta_m\) be the coefficients of complete subcircuits whose terms are compatible with circuit input \(\Lambda\). The circuit will evaluate to \(\theta_1 + \ldots + \theta_m\) under input \(\Lambda\).
\end{lemma}

\begin{proof}
Given Lemma~\ref{lem:induction}, we will use induction on the circuit structure. The base case is a leaf circuit node (indicator or parameter). The lemma holds trivially in this case. The inductive case is when \(n\) is an internal circuit node with children \(c_1, \ldots, c_k\). Suppose the lemma hold for these children. If \(n\) is a \(*\)-node, the lemma holds for \(n\) by decomposability (the complete subcircuits of \(n\) correspond to the Cartesian product of complete subcircuits for its children). If \(n\) is a \(+\)-node, the lemma holds for \(n\) since the complete subcircuits of \(n\) correspond to the union of complete subcircuits of its children.
\end{proof}

\begin{corollary}\label{coro:sum}
Let \(\theta_1, \ldots, \theta_m\) be the coefficients of \(\x\)-subcircuits in a decomposable and smooth arithmetic circuit \(\ac(\X)\). We then have \(\ac(\x) = \theta_1 + \ldots + \theta_m\).
\end{corollary}

We are now ready for the key result we are after.
\begin{theorem}\label{theo:polynomial}
Consider an arithmetic circuit \(\ac(\X)\) that computes factor \(f(\X)\). If the circuit is decomposable and smooth, then it also computes the marginals of factor \(f(\X\)).
\end{theorem}

\begin{proof}
Consider an instantiation \(\y\) of some variables \(\Y \subseteq \X\), let \(\x_1, \ldots, \x_m\) be all instantiations of variables \(\X\) that are compatible with \(\y\), and let \(\Lambda\) be the circuit input corresponding to these instantiations. Let \(\theta_i\) be the sum of coefficients of all \(\x_i\)-subcircuits. Since the circuit computes factor \(f(\X)\), we have \(\ac(\x_i) = f(\x_i)\) and, hence, \(f(\x_i) = \theta_i\) by Corollary~\ref{coro:sum}. By Lemma~\ref{lem:sum}, the circuit evaluates to  \(\theta_1 + \ldots + \theta_m\) under input \(\Lambda\), which is \(f(\x_1) + \ldots + f(\x_m) = f(\y)\). Hence, the circuit computes the factor marginals. 
\end{proof}

This theorem justifies the standard algorithm for computing marginals on arithmetic circuits, in linear time, as proposed in~\cite{Darwiche03}---that is, by simply evaluating the circuit as in Definition~\ref{def:ac}. In that work, however, the property of determinism was also assumed (discussed in the next section). Determinism is not necessary though for computing marginals as initially observed in \cite{PoonD11}.\footnote{\cite{PoonD11} introduced sum-product networks (SPNs), which are equivalent to decomposable and smooth ACs. More precisely, each can be converted into the other in linear time and space \cite{RooshenasL14}. The conversion is straightforward and amounts to adjusting for graphical notation.}  Our proof above uses different tools than those used in \cite{PoonD11} and is set in a more general context. Moreover, these tools and associated lemmas turn out to be essential for the rest of our treatment on the role of determinism, which we discuss in the next section.

As for the necessity of smoothness, consider the circuits \(\ac_1(A,B) = \lambda_a \lambda_b + \lambda_{\n(a)}\) and \(\ac_2(A,B) = \lambda_a \lambda_b + \lambda_{\n(a)} (\lambda_b + \lambda_{\n(b)}).\) Both circuits are decomposable and compute the same factor: \(f(a,b) = 1\), \(f(a,\n(b)) = 0\), \(f(\n(a),b) = 1\) and \(f(\n(a),\n(b)) = 1\). However, circuit \(\ac_1\) is not smooth while \(\ac_2\) is smooth.  Only \(\ac_2\) is guaranteed to compute factor marginals by Theorem~\ref{theo:polynomial}. For example, evaluating \(\ac_1\) at instantiation \(\n(a)\) gives \(\ac_1(\n(a)) = 0 \cdot 1 + 1 = 1 \neq f(\n(a)) = 2\) according to Definition~\ref{def:ac}, while \(\ac_2(\n(a)) = f(\n(a))\).\footnote{Theorem~1 of~\cite{FriesenD16} implies that factor marginals can be computed in time linear in the size of an arithmetic circuit, when the circuit is decomposable but not smooth. This complexity is not justified (but assumed) in the proof of the theorem. In fact, we are unaware of any justified algorithm that attains this complexity without smoothness; see also Footnote~\ref{footnote:smooth}.}

Before we discuss determinism, we note that decomposability and determinism were exploited recently in tractable, propositional reasoning within a semi-ring setting; initially in \cite{Kimmig12,KimmigJAL16}, then followed by \cite{FriesenD16}.

\section{Determinism} \label{sec:determinism}

The property of {\em determinism} \cite{darwiche01tractable} was employed for probabilistic reasoning in~\cite{Darwiche03} by compiling Bayesian networks into arithmetic circuits that are deterministic. It was also enforced by the algorithm of \cite{LowdD08} for learning arithmetic circuits. The property was later called {\em selectivity} in the works on SPNs, initially in \cite{peharzlearning}.

Using the terminology of our current formulation, the original definition of determinism would amount to this: An arithmetic circuit  is deterministic iff the terms of each two of its complete subcircuits are conflicting. We will adopt a weaker definition, which allows conflicting subcircuits as long as at most one of them has a non-zero coefficient.
\begin{definition}[Determinism]
An arithmetic circuit \(\ac(\X)\) is \de{deterministic} iff each \(+\)-node has at most one non-zero input when the circuit is evaluated under any instantiation \(\x\) of variables \(\X\).
\end{definition}

As mentioned earlier, the original proposal for using arithmetic circuits as a tractable representation of probability distributions~\cite{Darwiche03} ensured that these circuits are deterministic, in addition to being decomposable and smooth. Moreover, several methods were proposed in~\cite{Darwiche03} for compiling Bayesian networks into ACs with these properties. One of these methods ensures that the size of the AC is proportional to the size of a jointree for the network. Another method yields circuits that can sometimes be exponentially smaller, and is implemented in the publicly available \ace\ system \cite{Chavira.Darwiche.Aij.2008}; see also \cite{UAIEvaluation08} for an empirical evaluation of this system in one of the UAI inference evaluations.

While determinism is not needed to compute factor marginals, it is needed for the correctness of the linear-time MPE algorithm of \cite{ChanD06}. This was missed in some earlier works \cite{PoonD11}, which used this algorithm on non-deterministic ACs (i.e., SPNs) without realizing that it is no longer correct. This oversight was noticed in later works \cite{Peharz16,MauaC17}.\footnote{\cite{Peharz16} proposed a polytime algorithm that converts an SPN into one that is deterministic and smooth (called an \emph{augmented SPN}), but this new SPN computes a different factor than the one computed by the original SPN. Hence, its MPEs cannot be generally converted into MPEs of the original SPN.} We next reveal the reason why computing MPE without determinism is hard. Later in the section, we reveal the reason why the MPE algorithm of \cite{ChanD06} fails without determinism.

The key observation is this. Consider variables \(\X\) which are partitioned into \(\Y\) and \(\Z\). Given a decomposable and smooth arithmetic circuit \(\ac(\X)\) that computes factor \(f(\X)\), one can obtain in linear time a decomposable and smooth \(\ac(\Y)\) that computes the projection \(\sum_\Z f(\X)\). This is achieved by simply setting all indicators \(\lambda_z\) to \(1\); see \cite{darwicheJACM-DNNF} for the root of this observation. Moreover, an MPE for the projection \(\sum_\Z f(\X)\) is a MAP for the original factor \(f(\X)\). Hence, a polytime MPE algorithm implies a polytime MAP algorithm on decomposable and smooth ACs. We know, however, that Naive Bayes networks have linear-size decomposable and smooth ACs, while MAP is hard on these networks \cite{Campos11}. Therefore, the existence of a polytime MPE algorithm on such circuits will contradict standard complexity assumptions. These observations can be abstracted into the following lemma, which succinctly and intuitively explains why MPE is not tractable on decomposable and smooth circuits.

\begin{lemma}
A circuit representation that supports projection and MPE in polytime also supports MAP in polytime.
\end{lemma}
Note that deterministic, decomposable and smooth ACs do not support projection in polytime so the above argument does not apply in this case (setting indicators \(\lambda_z\) to \(1\) will generally destroy determinism).

More formally, let \(\ac(\X)\) be a decomposable and smooth arithmetic circuit that computes a factor \(f(\X)\).  For a given value \(k\), consider the decision problems:
\begin{itemize}
\item \textbf{D-MPE-AC}: Is there an instantiation \(\x\) where \(\ac(\x) > k\)?
\item \textbf{D-MAP-AC}: Is there an instantiation \(\y\) where \(\sum_\z \ac(\y,\z) > k\)? (\(\X\) is partitioned into \(\Y\) and \(\Z\)).
\end{itemize}

We now have the following result, whose proof expands the one given in~\cite{Peharz16} for SPNs based on the above observations; 
see also~\cite{MauaC17} for an in-depth discussion of MPE hardness on SPNs.
\begin{theorem}\label{theo:mpe-1}
The problem D-MPE-AC is NP-complete.
\end{theorem}

\begin{proof}
Given instantiation \(\x\) and value \(k\), we can test whether \(f(\x) > k\) by evaluating the circuit \(\ac\) in time linear in the size of the circuit.  Hence, the problem is in NP.  To show NP-hardness, we reduce the (decision) problem of computing (partial) MAP in a naive Bayes network, which is NP-complete \cite{Campos11}, to MPE in a decomposable and smooth arithmetic circuit.  Suppose we have a naive Bayes network with a root node \(X_0\) and leaf nodes \(\X\), and inducing a distribution \(\pr(X_0,\X)\).  We can compile this network into a polysize decomposable, deterministic and smooth arithmetic circuit \(\ac_0(X_0,\X)\) that computes \(\pr(X_0,\X)\), e.g., as in \cite{Chavira.Darwiche.Aij.2008}.  We can sum-out variable \(X_0\) in the circuit \(\ac_0(X_0,\X)\) by setting the indicators of \(X_0\) to one.  The resulting circuit \(\ac_1(\X)\) is decomposable and smooth, and computes the (marginals of) factor \(\sum_{X_0} \pr(X_0,\X)\).  For a given value \(k\), there exists an instantiation \(\x\) where \(\ac_1(\x) > k\) iff there exists an instantiation \(\x\) where \(\sum_{x_0} \pr(x_0,\x) > k\), which is an NP-complete problem \cite{Campos11}.
\end{proof}

\begin{corollary}\label{cor:map}
The problem D-MAP-AC is NP-complete.
\end{corollary}

The following lemma reveals the precise role of determinism, which stands behind the correctness of the linear-time MPE algorithm of \cite{ChanD06}. It basically shows a one-to-one correspondence between the non-zero rows of the factor computed by a circuit and the complete subcircuits with non-zero coefficients. 
\begin{lemma} \label{lem:determinism}
Consider an arithmetic circuit \(\ac(\X)\) that computes factor \(f(\X)\) and is deterministic, decomposable and smooth (hence, can be viewed as a collection of \(\x\)-subcircuits). For each instantiation \(\x\), we have:
\begin{itemize}
\item[(a)] If the circuit has two distinct \(\x\)-subcircuits, one of them must have a zero coefficient.
\item[(b)] If \(f(\x) > 0\), the circuit contains a unique \(\x\)-subcircuit with coefficient \(f(\x)\).
\end{itemize}
\end{lemma}

\begin{proof}
To prove (a), suppose the circuit contains two distinct \(\x\)-subcircuits \(\alpha_1\) and \(\alpha_2\) that have non-zero coefficients. We will now establish a contradiction. Since \(\alpha_1\) and \(\alpha_2\) are distinct, each \(\alpha_i\) must include a distinct child \(c_i\) of some \(+\)-node in the circuit. If we evaluate the circuit at instantiation \(\x\), both \(c_1\) and \(c_2\) will have non-zero values. Hence, the circuit cannot be deterministic, which is a contradiction. To prove (b), suppose \(f(\x) > 0\) and let \(\alpha_1, \ldots, \alpha_m\) be all \(\x\)-subcircuits. At most one \(\alpha_i\) can have a non-zero coefficient by~(a). Since the circuit computes the factor, it must evaluate to \(f(\x)\) under instantiation \(\x\). Hence, exactly one \(\alpha_i\) has \(f(\x)\) as its coefficient.
\end{proof}

Lemma~\ref{lem:determinism} allows us to prove the correctness of the MPE algorithm given by \cite{ChanD06} under the more general setting we have in this paper.  This original algorithm is based on converting a deterministic, decomposable and smooth AC that computes a distribution \(\pr(\X)\) into a \emph{maximizer circuit.} Evaluating this circuit under evidence \(\y\), as in Definition~\ref{def:ac}, gives the MPE value \(\argmax_{\x \sim \y} \pr(\x)\).

An arithmetic circuit \(\ac(\X)\) is converted into a maximizer circuit, denoted \(\ac_{\max}(\X)\), by replacing every \(+\)-node with a \(\max\)-node.  The complete subcircuits of \(\ac_{\max}\) are defined as in Definition~\ref{def:subcircuit}, but where exactly one child of each visited \(\max\)-node is selected.

\begin{theorem} \label{theo:mpe}
Let \(\ac(\X)\) be a deterministic, decomposable and smooth arithmetic circuit that computes a factor \(f(\X)\) and let \(\ac_{\max}(\X)\) be its maximizer circuit. Then \(\ac_{\max}(\y) = \max_{\x \sim \y} f(\x)\) for \(\Y \subseteq \X\).
\end{theorem}

\begin{proof}
By Lemma~\ref{lem:determinism}, there is a one-to-one correspondence between the non-zero rows of factor \(f(\X)\) and \(\x\)-subcircuits with non-zero coefficients.  Let \(\theta_1,\ldots,\theta_m\) be the coefficients of \(\x\)-subcircuits, where \(\x\) is compatible with \(\y\). Hence,  \(\max \{\theta_1,\ldots,\theta_m\} = \max_{\x \sim \y} f(\x)\). That is, the MPE value is a coefficient of some \(\x\)-subcircuit---call it an MPE-subcircuit. We will think of the algorithm as composing an MPE-subcircuit in addition to computing its coefficient and show that \(\ac_{\max}(\y)=\max \{\theta_1,\ldots,\theta_m\}\) by induction on the circuit structure (see Lemma~\ref{lem:induction}). The base case trivially holds for leaf circuit nodes (indicators and parameters).  Assume \(n\) is an internal circuit node and the above equality holds for its children \(c_1,\ldots,c_k\) having MPE-subcircuits \(\alpha_i\) and coefficient \(\eta_i\). If \(n\) is a \(*\)-node, then by decomposability, an MPE-subcircuit for \(n\) can be found by joining \(\alpha_1,\ldots,\alpha_k\) with \(\eta_1 * \ldots * \eta_k\) as its coefficient. If \(n\) is a \(\max\)-node, then by determinism, an MPE-subcircuit for \(n\) can be found from the \(\alpha_i\) with the largest \(\eta_i\) with \(\max_{i=1}^k \eta_i\) as its coefficient.
\end{proof}

Once a maximizer circuit is evaluated to \(\theta\), one can identify an \(\x\)-subcircuit that has \(\theta\) as its coefficient, with \(\x\) being an MPE instantiation; see \cite{ChanD06}.\footnote{Smoothness is not strictly needed for this algorithm, except that it ensures that a full variable instantiation is returned.}

Without determinism, a circuit may have multiple \(\x\)-subcircuits for a given \(\x\), each having a non-zero coefficient. By Corollary~\ref{coro:sum}, the value of \(\x\), \(\ac(\x)=f(\x)\), is the sum of these coefficients. An MPE algorithm that does not perform this sum cannot be correct.\footnote{This MPE algorithm was used on selective SPNs (equivalent to deterministic and decomposable ACs) in~\cite{Peharz16}. It was also adapted to algebraic model counting (AMC) in~\cite{KimmigJAL16} and to Sum-Product Functions (SPFs) in~\cite{FriesenD16}. Determinism was not required in \cite{KimmigJAL16}. This is sound since AMC problems correspond to Boolean circuits where the weight of an instantiation is a product of literal weights, and is independent of how many times the instantiation appears as a subcircuit.}

Before we further discuss the impact of relaxing determinism, we point to a new class of arithmetic circuits, the Probabilistic Sentential Decision Diagram (PSDD)~\cite{KisaKR14}, which imposes \emph{stronger} versions of decomposability and determinism.  This enables the multiplication of two ACs in polytime, which is otherwise hard under the standard versions of these properties \cite{ShenCD16}.

\section{The Impact of Relaxing Determinism} \label{sec:relax}

We now consider two new implications of relaxing determinism, one positive and one negative. We also address an apparent paradox: How could a representation (decomposable and smooth ACs) allow the computation of marginals easily (a PP-complete problem), yet not allow the computation of MPE easily (an NP-complete problem)? Recall that the complexity class NP is included in the class PP.

The positive implication is that relaxing determinism can lead to exponentially smaller arithmetic circuits.
\begin{theorem}[Separation]\label{theo:separation}
There is a family of factors \(f_n(\X_n)\) where 
(1) there exists a decomposable and smooth arithmetic circuit \(\ac_n(\X_n)\) that computes the marginals of \(f_n\), with a size polynomial in \(n\);
(2) every deterministic, decomposable and smooth circuit that computes the marginals of factor \(f_n\) must have a size exponential in \(n\).
\end{theorem}

\begin{proof}
\cite{BovaCMS16} identifies a family of Boolean functions (the Sauerhoff functions) \(S_n\) that have decomposable NNFs (DNNFs) with sizes polynomial in \(n,\) but where their deterministic DNNFs (d-DNNFs) must have sizes exponential in \(n\).  Previously known separations were conditional on the polynomial hierarchy not collapsing \cite{darwicheJAIR02}, but \cite{BovaCMS16} does not make such an assumption (and neither do we).

Let \(g_n\) denote a polysize DNNF for function \(S_n\) and let \(\ac_n\) denote the polysize decomposable and smooth arithmetic circuit obtained by: replacing the inputs of \(g_n\) with the corresponding indicator variables, replacing conjunctions and disjunctions by products and sums, respectively, then smoothing if necessary.  The resulting arithmetic circuit \(\ac_n\) has a positive value (possibly \(> 1\)) on input \(\x\) iff the original function \(S_n\) evaluates to true.  We now show that if \(f_n\) is the factor computed by arithmetic circuit \(\ac_n\), then any deterministic, decomposable and smooth AC that computes \(f_n\) must have an exponential size.

Let \(\ac^\prime_n\) be such a circuit.  Consider the d-DNNF \(g^\prime_n\) obtained by: replacing the indicator variables with the corresponding literals of variables \(\X\), replacing products and sums with conjunctions and disjunctions, respectively, and by replacing positive parameters with true and zero parameters with false.  Note that \(g^\prime_n(\x)\) is true iff \(\ac^\prime_n(\x) > 0\), i.e., a complete subcircuit for \(g^\prime_n\) evaluates to true iff the corresponding subcircuit for \(\ac^\prime_n\) has a positive coefficient.  Hence, if \(\ac^\prime_n\) had a sub-exponential size, then function \(S_n\) would have a sub-exponentially sized d-DNNF, which we know does not exist \cite{BovaCMS16}.
\end{proof}

We now get to a newly identified, negative implication of relaxing determinism. It pertains to compiling ACs from probabilistic models and requires the following notion.

\begin{definition}
A set of parameters \(\Theta\) is \de{complete} for factor \(f(\X)\) iff for every instantiation \(\x\), the parameter \(f(\x)\) can be expressed as a product of parameters in \(\Theta\). 
\end{definition}
The parameters of a Bayesian network are complete for its distribution; those of a Markov network are complete for its partition function; and the parameters \(\Theta = \{0,1\}\) are complete for {\em Boolean factors}: \(f(\X)\) with \(f(\x) \in \{0,1\}\). 

We will write \(\ac(\X,\Theta)\) to denote an arithmetic circuit whose parameters are in \(\Theta\). The following theorem states a key property which is lost when relaxing determinism.
\begin{theorem}[Completeness]\label{theo:determinism}
Consider factor \(f(\X)\) and complete parameters \(\Theta\). There must exist an arithmetic circuit \(\ac(\X,\Theta)\) that computes the factor marginals and is deterministic, decomposable and smooth.
\end{theorem}

\begin{proof}
Consider the factor polynomial \(\sum_\x f(\x) \prod_{x \sim \x} \lambda_x\) in Definition~\ref{def:poly} and replace each \(f(\x)\) by a product of parameters from \(\Theta\). The result can be represented by an AC that is deterministic, decomposable and smooth.
\end{proof}

The standard methods for compiling Bayesian networks, and graphical models more generally, into arithmetic circuits do indeed limit the circuit parameters to those appearing in the model factors. Hence, the compilation process amounts only to finding a (small) circuit structure since the circuit parameters are already predetermined. As mentioned earlier, these methods can yield relatively small circuits for some graphical models with very high treewidth \cite{Chavira.Darwiche.Aij.2008,UAIEvaluation08}.

The above property is lost if one insists on constructing arithmetic circuits that are decomposable and smooth, but not deterministic. This is shown in the following theorem, which refers to \emph{dead} circuit nodes: ones that appear only in complete subcircuits that have zero coefficients.\footnote{Dead nodes can be replaced by the constant zero without changing the factor computed by the circuit. One can relax determinism trivially by adding dead nodes, but that does not help as far as obtaining smaller circuits.}

\begin{theorem}[Parametric Incompleteness]\label{theo:incompleteness-bool}
Let \(f(\X)\) be a Boolean factor and \(\Theta = \{0,1\}\) (\(\Theta\) is complete for \(f\)).  A circuit \(\ac(\X,\Theta)\) cannot compute \(f(\X)\) if it is decomposable, smooth and free of dead nodes, but not deterministic.
\end{theorem}

\begin{proof}
If the AC has no \(+\)-node, then it is vacuously deterministic. Otherwise, it has a \(+\)-node.  Since the circuit is not deterministic, there is a \(+\)-node \(n\) that violates determinism.  This node is included in some complete subcircuit with a non-zero coefficient (otherwise, the node \(n\) is dead).  Since node \(n\) violates determinism, we can find two distinct \(\x\)-subcircuits, with non-zero coefficients, that differ by the branch selected at node \(n\).  Since the circuit computes factor \(f(\X)\), Lemma~\ref{lem:sum} implies that the coefficients of these \(\x\)-subcircuits must add up to \(f(\x) = 1\).  There must then exist an \(\x\)-subcircuit whose coefficient is in \((0,1)\), exclusive, i.e., the circuit has a parameter not in \(\{0,1\}\).
\end{proof}

In other words, if a decomposable and smooth circuit \(\ac(\X,\{0,1\})\) computes the marginals of a Boolean factor, it must also be non-trivially deterministic. This result has a major implication on compiling probabilistic graphical models into ACs that are not deterministic. That is, one cannot generally restrict the circuit parameters to those appearing in the model, otherwise a circuit may not exist.

Therefore, while relaxing determinism can lead to exponentially smaller circuits, finding these circuits is now more involved as it may require searching for parameters. This demands new techniques as all techniques we are aware of for compiling models into deterministic circuits assume that the circuit parameters come from model parameters. 

Our last contribution relates to the following apparent paradox. Suppose we have a set of factors \(f_1(\X_1),\ldots,f_n(\X_n)\), representing a probabilistic graphical model that has a corresponding joint factor \(f = f_1 \cdots f_n\). Consider now the following decision problems, over such probabilistic graphical models, which correspond to computing the MPE and marginals:
\begin{itemize}
\item \textbf{D-MPE}: Is there an instantiation \(\x\) where \(f(\x) > k\)?
\item \textbf{D-PR}: Is \(\sum_{\x} f(\x) > k\)?
\end{itemize}
D-MPE is NP-complete, whereas D-PR is PP-complete. Moreover, the complexity class PP includes NP. Yet, decomposable and smooth ACs allow one to compute marginals in linear time, while computing MPE, which is no harder, is hard on these circuits!

To resolve this apparent paradox, one must observe the sometimes subtle distinction between a {\em representation} and the {\em computation} needed to produce that representation. The representation here is decomposable and smooth ACs, and the computation is the algorithm used to compile a graphical model into this representation. While the representation itself does not facilitate the computation of MPE, the compilation algorithm must be sufficient to compute the MPE query without additional complexity (beyond polynomial). To formalize this, we need the following lemma.

\begin{lemma}\label{lemma:mpe}
D-MPE can be reduced to D-PR.
\end{lemma}

We now have the following result, which implies that a polytime compilation algorithm for decomposable and smooth ACs can be used as a sub-routine for a polytime algorithm for computing MPEs (which we do not expect to exist, under typical complexity theoretic assumptions).
\begin{theorem}\label{theo:mpe-2}
Consider an algorithm \(\Xi\) that takes a set of factors \(f_1(\X_1), \ldots, f_n(\X_n)\) as input and returns a decomposable and smooth arithmetic circuit that computes the marginals of factor \(f = f_1 \cdots f_n\). Let \(s\) be the size of input factors and let \(O(t(s))\) be the time complexity of algorithm \(\Xi\). One can compute the MPE of factor \(f\) in time \(O(t(\mathrm{poly}(s)))\).
\end{theorem}

These findings highlight an interesting property of decomposable and smooth ACs. They ``store'' the results to an exponential number of marginal queries, where each result can be retrieved by a simple traversal of the circuit. Yet, they do not ``store'' the answers to MPE queries, even though these queries are easier. The implication of this can be seen from two angles, depending on whether these circuits are compiled from models or learned from data. In the former case, the compilation algorithm is readily available to answer MPE queries, but at the cost of invoking this algorithm for each query. In the latter case, however, answering MPE queries remains a challenge. Hence, learning circuits that are not deterministic needs to yield an additional benefit that compensates for this loss in tractability. This could be a simpler learning algorithm; a smaller learned circuit; or a learned circuit whose factor is superior from a statistical learning viewpoint.

\section*{Acknowledgements} 

This work has been partially supported by NSF grant \#IIS-1514253, ONR grant \#N00014-15-1-2339 and DARPA XAI grant \#N66001-17-2-4032. We thank YooJung Choi, Umut Oztok, Yujia Shen, and Guy Van den Broeck for comments and discussions on this paper.

\appendix

\section{Proofs}

\begin{proof}[Proof of Corollary~\ref{cor:map}]
Given an instantiation \(\y\) and value \(k\), we can test whether \(\sum_\z \ac(\y,\z) > k,\) by evaluating the circuit \(\ac\) in time linear in the size of the circuit.  Hence, the problem is in NP.  To show the problem is NP-hard, we reduce the (decision) problem of computing (partial) MAP in a naive Bayes network, as in the proof of Theorem~\ref{theo:mpe-1}.
\end{proof}

\begin{proof}[Proof of Lemma~\ref{lemma:mpe}.]
We first reduce D-MPE to satisfiability on a CNF (this is essentially an instance of Cook's theorem).  Satisfying assignments of the CNF correspond to solutions of the D-MPE query, which we can count using a D-PR query.

First, we construct a Boolean circuit that takes inputs \(\x\), and outputs true if \(f(\x) > k\) and false otherwise.  We construct a (multiplexer) circuit for each factor \(f_i(\X_i)\), which has inputs \(\x_i\) and outputs a bitstring \(\y_i\) representing a bit encoding of the value \(f_i(\x_i)\) (which we assume are rational values).  We then construct a circuit that represents a multiplier, which takes as input the bitstrings \(\y_i\) and outputs another bitstring \(\z\) representing the product of \(f_i(\x_i)\).  Finally, we have another circuit that takes the bitstring \(\z\) as input, and outputs true if this bitstring represents a value that is greater than \(k\), and false otherwise. Hence, the output of this circuit is true iff \(f(\x) > k\).  Each one of the constructed circuits has size polynomial in the size of the inputs, i.e., the aggregate size of the factors and the number of bits needed to represent their values.

We can reduce this circuit to a CNF by adding auxiliary variables \(\Y\), using one new variable for the output of each gate, i.e., we reduce circuit satisfiability to 3-SAT; see, e.g., \cite{KleinbergTardos}.  This results in a set of Boolean factors \(g_1,\ldots,g_m\).  If \(g = g_1 \cdots g_m\), then \(g(\x) = \sum_{\y} g(\x,\y) > 0\) iff \(f(\x) > k\), and \(\sum_{\x,\y} g(\x,\y) > 0\) iff there exists an input \(\x\) where \(f(\x) > k\).
\end{proof}

\begin{proof}[Proof of Theorem~\ref{theo:mpe-2}.]
Given factors \(f_1,\ldots,f_n\) of size \(s\), we first construct a set of factors \(g_1,\ldots,g_m\) of size \(\mathrm{poly}(s)\), as in Lemma~\ref{lemma:mpe}.  We invoke algorithm \(\Xi\) on factors \(g_1,\ldots,g_m\), obtaining a decomposable and smooth arithmetic circuit \(g\) representing \(g_1 \cdots g_m\) in time \(O(t(\mathrm{poly}(s)))\).  The size of \(g\) is also \(O(t(\mathrm{poly}(s)))\) (the size of the circuit cannot be larger than the time needed to construct it).  The same amount of time is required to evaluate the marginal of \(g\), hence the overall time to compute the MPE of factor \(f\) is \(O(t(\mathrm{poly}(s)))\).
\end{proof}

{\small
%\bibliography{bib/references}
%\bibliographystyle{icml2017}

}

\end{document}